\documentclass[10pt,english]{article}
\pdfoutput=1

\usepackage[]{fontenc}
\usepackage{graphicx}
\usepackage{subfig}

\makeatletter

\usepackage{amstext}
\usepackage{amsmath}
\usepackage{amsfonts}
\usepackage{amsmath}
\usepackage{amsfonts}
\usepackage{amssymb}

\def\algge{{\mathrm{ge}_{1}}}
\def\algmax{\max_{10}}
\def\algrnd{\mathrm{rnd}_{10}}
\newcommand{\paleo}{\textsc{Paleo}}
\newcommand{\courses}{\textsc{Courses}}
\newcommand{\retail}{\textsc{Retail}}

\newtheorem{theorem}{Theorem}

\newtheorem{proposition}{Proposition}
\newtheorem{lemma}{Lemma}
\newtheorem{problem}{Problem}
\newtheorem{example}{Example}
\newtheorem{definition}{Definition}

\newcommand{\qed}{\hfill \ensuremath{\Box}}

\newenvironment{proof}[1][Proof]{\begin{trivlist}
\item[\hskip \labelsep {\bfseries #1}]}{\end{trivlist}}

\makeatother

\usepackage{babel}

\begin{document}

\title{Multiple Hypothesis Testing in Pattern Discovery}

\author{Sami Hanhij\"arvi \ \ \ \ Kai Puolam\"aki \ \ \ \  Gemma C. Garriga \\
Helsinki Institute for Information Technology HIIT \\
Department of Information and Computer Science\\
Helsinki University of Technology, Finland \\
{\tt firstname.lastname@tkk.fi}}

\date{}

\maketitle

\thispagestyle{empty}
\begin{abstract}
The problem of multiple hypothesis testing arises when there are more
than one hypothesis to be tested simultaneously for statistical significance.
This is a very common situation in many data mining applications.
For instance, assessing simultaneously the significance of all frequent
itemsets of a single dataset entails a host of hypothesis, one for
each itemset. A multiple hypothesis testing method is needed to control
the number of false positives (Type I error). Our contribution in
this paper is to extend the multiple hypothesis framework to be used
with a generic data mining algorithm. We provide a method that provably
controls the family-wise error rate (FWER, the probability of at least
one false positive) in the strong sense. We evaluate the performance
of our solution on both real and generated data. The results show
that our method controls the FWER while maintaining the power of the test.

\vspace{.3cm}
\end{abstract}

\section{Introduction}

This paper addresses the problem of assessing the statistical significance
of the patterns produced by data mining algorithms. In traditional
statistics the issue of significance testing has been thoroughly studied
for many years. Given observed data and a structural measure (namely,
test statistic) calculated from the data, a hypothesis testing method
can be used to decide whether the observed data was drawn from a given
null hypothesis. Under this framework, randomization approaches help
in producing multiple random datasets sampled from a specified null
hypothesis. If the test statistic of the original data deviates significantly
from the test statistics of the random datasets, then the null hypothesis
can be discarded and the result can be considered significant.

Recently, there has been an increasing interest in randomization techniques
for data mining (e.g.~\cite{Gionis06,Ojala08sdm}). For example,~\cite{Gionis06}
introduced a method to sample 0--1 matrices uniformly at random such
that the row and column margins of a matrix are preserved. The method
is extended by~\cite{Ojala08sdm} for real-valued matrices. As in
the traditional framework from statistics, the randomized samples
can be interpreted to be drawn from a null distribution and they are
used to test the statistical significance of discovered patterns.
In the case of 0--1 matrices, for example, the $p$-value of a frequent
set could be defined as the fraction of randomized datasets that have
a higher frequency for the set than with the original data.

The statistical significance testing problem is well understood when
the hypothesis to be tested are known in advance and the number of
hypothesis is fixed (see \cite{Dudoit03,Lehmann59,Westfall93}). In
the simplest case, there is only one hypothesis (such as a frequency
of a given frequent set) and the statistical hypothesis test controls
the probability of false positives, also called Type I error. A proper
statistical significance test of level $\alpha$ (typical choices
being $\alpha\in\{0.01,0.05\}$) falsely declares a pattern that follows
the null distribution as significant (false positive) with a probability
of at most $\alpha$.

The problem of \emph{multiple hypothesis testing} arises when there
are more than one hypothesis to test simultaneously. This is a very
common situation in data mining: for instance, an algorithm for frequent
set mining typically outputs a collection of itemsets whose frequency
is above a user-specified threshold in the data. As a simple example,
assume that we have 1000 independent patterns that all follow the
null hypothesis (are random effects in the data). A naively applied
statistical significance test of level $\alpha=0.05$ is likely to
falsely declare about 50 of these 1000 patterns as significant, even
though all of the measurements obey the null distribution. To remedy
the independent evaluation of hypothesis, the theory of multiple hypothesis
testing assesses a multiple comparison problem, that is, considers
simultaneously a family of statistical inferences. 

There exists traditional methods in statistics to tackle the problem
of multiple hypothesis testing, of which the Bonferroni correction
is the simplest and probably the best known. These methods vary with
respect to the power, type of error they control and the assumptions
they make of the dependency structure within the data. A common property
for all of these methods is that as the number of hypothesis to be
tested increases, the methods lose power, that is, they are less likely
to find the hypothesis that are not from the null distribution.

A data mining algorithm can produce a host of patterns. For example,
in the frequent set mining, the number of possible frequent sets is
exponential in the number of attributes. If each pattern is considered
as a separate hypothesis, then a direct application of multiple hypothesis
testing would be too naive: the method would not declare any pattern
as significant due to the large number of patterns to be tested.

A possible solution to overcome this problem consists of limiting
the hypothesis space. For example, in frequent set mining one could
only consider frequent sets of at most the given length. Another limiting
approach was proposed by \cite{Webb06} for the specific application
of association rule mining, where the data is first split into two
folds. The first half of the data is used to run the data mining algorithm
to define the hypothesis (association rules), and the second half
of the data is then used to test for the significance of those patterns
using some of the known multiple hypothesis testing methods. Such
an approach works when the data can be split into two halves that
are independent one of the other, and also, when the algorithm can be run
on partial data. However, these conditions are not feasible for all applications:
consider for example finding patterns such as frequent subgraphs from
a network, which cannot be trivially split into two independent components.

A completely different approach for multiple hypothesis testing in
association rule mining was proposed by~\cite{Lallich06}. Their
idea is to use bootstrap to find an upper bound for the deviation
of the test statistic between random and original data, such that
it controls the probability of falsely declaring a pattern significant.
All association rules mined from the original data that have a larger
test statistic deviation from its mean than the chosen threshold,
will be declared significant. This method has the same limitations
as~\cite{Webb06} in that it can only be used for association rule mining.
Furthermore, bootstrapping transactions does not break the dependency
between antecedent and consequent. This is of course a choice of a
null hypothesis, but it may not make sense in all association rule
mining contexts.

Our contribution in this paper is to provide a proper definition of
$p$-value for patterns using the randomized samples, and show that
with this $p$-value the known multiple hypothesis testing methods
can be used directly on the patterns output by a generic data mining
algorithm, regardless of the potentially large number of possible
patterns. We make no assumption on the data. The main contributions
of this paper are: the definition of a $p$-value suitable for 
data mining applications; a general method to assess significant
patterns using a valid statistical testing methodology; and experimental
verification of the validity and power of the presented method.

The paper is organized as follows. In Section \ref{sec:problemstatement}
we provide the problem statement and essential definitions and in
Section \ref{sec:main} we state our contribution without any formal
proof. In Section, \ref{sec:theory} we give a summary of multiple
hypothesis testing and prove the validity of our method; this section
can be skipped in the first reading. Section \ref{sec:relatedwork}
reviews the related methods, Section \ref{sec:experiments} contains
experiments and the paper ends with the discussion in Section
\ref{sec:conclusion}.

\section{Formal problem statement}

\label{sec:problemstatement}

We consider the general case where we have a \emph{data mining algorithm}
$A$ that, given an input dataset $D$, outputs a set of patterns
$P$, or $A(D)=P$. The set $P$ is a subset of a universe of all
patterns ${\cal \mathcal{P}}$. For different input datasets, the algorithm
may output a different set of patterns, still from ${\mathcal{P}}$.
We further assume defined a \emph{test
statistic} $f(x,D)\in{\mathbb{R}}$, associated to an input pattern
$x\in{\cal P}$ for the dataset $D$; large values of the statistic
are assumed to be more interesting for the user. We assume that we
have at our disposal a randomization algorithm with which one can
sample $n$ datasets i.i.d. from the \emph{null distribution} $\Pi_{0}$
corresponding to the null hypothesis $H_{0}$%
\footnote{Usually a null distribution would be defined for the test statistic
when the null hypothesis holds. In our case, the distribution of the
datasets, together with $A$ and $f$, defines the null distribution
for the test statistic.}. 
Our intuition is that if the test statistic for a given pattern
$x$ is an extreme value in the null distribution, then we can declare
the pattern significant. 
We denote the datasets sampled from the null distribution by $D_{i}$,
where $i\in[n]$, and $[n]=\{1,\ldots,n\}$. 

Using the above definitions, we can define our problem as follows. 

\begin{problem}
\label{prob:uno} 
Given a data mining algorithm $A$, a dataset $D$,
a test statistic $f$ and a null distribution $\Pi_{0}$, which of
the patterns output by $A(D)$ are statistically significant?
\end{problem}

In this work we apply our method to frequent itemset mining and association
rule mining from 0--1 data, and also, frequent subgraph mining from
networks. However, our formulation is general and, unlike much of
the previous work, we do not restrict ourselves to any particular
types of data nor patterns. 

\begin{example}
In frequent itemset mining the dataset $D$ could be a 0--1 data
matrix, the set of all possible patterns $\mathcal{P}$ could be all
subsets of attributes (itemsets), the algorithm $A$ could be a level-wise
algorithm with a given frequency threshold and the test statistic
$f(x,D)$ could be the frequency of the itemset $x$ in data matrix
$D$. The null distribution $\Pi_{0}$ could be the uniform distribution
over all binary matrices of the same size with fixed row and column
margins; datasets from this null distribution can be sampled using
the swap randomization presented in \cite{Gionis06}. Our objective
would be to decide which of the frequent itemsets output by the algorithm
$A$ are statistically significant. 
\end{example}

The methods of statistical significance testing often make assumptions
about the shape of the null distribution (e.g., that the statistics 
follow a normal distribution). We do not make such assumptions, but 
we require that the algorithm $A$ satisfies the {\it minP-property},
which will be defined later in Section~\ref{sect:minp}.

\section{Main contribution: A significance testing method}

\label{sec:main}

In this section we state succinctly the main contribution of this
paper, that is, a method to test the significance of patterns within
the framework discussed in Section \ref{sec:problemstatement}. The
detailed discussion with derivations and references are presented
in Section~\ref{sec:theory}, and the experimental results are presented
in Section~\ref{sec:experiments}.

We first define two empirical $p$-values: the first one is the sample-based
empirical $p$-value in Definition~\ref{def:samplebased}, which 
weights each randomized dataset equally; the second is the pool-based $p$-value
in Definition~\ref{def:poolbased},
where the patterns obtained from the randomized datasets are weighted equally. 

\begin{definition}[Sample based empirical p-value] Let $D$ be our original 
dataset, $D_i$ for $i \leq n$ be the $n$ datasets sampled from the null distribution
and $D_{n+1}=D$. Let also $f(x,D)$ be the test statistic associated to an input pattern 
$x \in \mathcal{P}$ returned by algorithm $A$. We define the sample-based
$p$-value as follows:
\begin{equation} 
p_{D}^{sample}(x)=\frac{\sum_{i=1}^{n+1}h(x,D,D_{i})}{n+1},\label{eq:sample}
\end{equation}
where,
\begin{equation}
h(x,D,D')=\left\{ \begin{array}{lcl}
\frac{\left|\{x'\in A(D')\mid f(x,D)\le f(x',D')\}\right|}{\left|A(D')\right|} & , & |A(D')|>0\\
0 & , & |A(D)'|=0\end{array}\right..
\label{eq:h}
\end{equation}
\label{def:samplebased}
\end{definition}

\begin{definition}[Pool based empirical p-value]
Let $D$ be our original 
dataset, $D_i$ for $i \leq n$ be the $n$ datasets sampled from the null distribution
and $D_{n+1}=D$. Let also $f(x,D)$ be the test statistic associated to an input pattern 
$x \in \mathcal{P}$ returned by algorithm $A$. We define the pool-based
$p$-value as follows:
\begin{equation}
p_{D}^{pool}(x)=\frac{\sum_{i=1}^{n+1}|\{y\in A(D_{i})|f(y,D_{i})\ge f(x,D)\}|}{\sum_{i=1}^{n+1}|A(D_{i})|}.
\label{eq:pool}
\end{equation}
\label{def:poolbased}
\end{definition}

The $p$-values of the sample-based method 
represent the probability that, given a random dataset from the
null distribution of datasets, a test
statistic has a more extreme value. The difference between the two
methods becomes from the weight to patterns. In the pool-based method,
each pattern of the output of any dataset is weighted equally. Therefore,
the datasets that result in more patterns have more control over
the $p$-value calculation. Conversely, the sample-based method treats
each dataset equally and the patterns in a single dataset share
the weight of the dataset uniformly.

Briefly, we denote by $p_{i},$ $i\in[m]$, where 
$[m]=\{1,\ldots,m\}$, and $m=|A(D)|$, the sorted empirical
$p$-values for the patterns $A(D)$ given by Equation (\ref{eq:sample})
or (\ref{eq:pool}), i.e., $p_{1}\le\dots\le p_{m}.$ We

The family-wise error rate (FWER) is defined as the probability of
falsely declaring at least one patten in $A(D)$ as significant, where
$A(D)$ represents the set of patterns output by algorithm $A$
using dataset $D$. A more formal definition will be provided in the
next section. To control the FWER at the level $\alpha$, we can
apply the Holm-Bonferroni method~\cite{Holm79}, to obtain the so-called
{\it adjusted $p$-values}. The equation
to compute the adjusted $p$-value of a pattern $x_{i}$ under Holm-Bonferroni
method is,
\begin{equation}
\tilde{p}_{i}^{H}=\left\{ \begin{array}{lcr}
\min(1,mp_{1}) & , & i=1\\
\min\left(1,\max(\tilde{p}_{i-1}^{H},(m-i+1)p_{i})\right) & , & i>1\end{array}\right..\label{eq:HB}
\end{equation}
Then, we declare the pattern $x$ significant if its adjusted $p$-value
satisfies $\tilde{p}_{i}^{H}\le\alpha$. Note that the Holm-Bonferroni
method is general and can be used with any definition of $p$-value
when the number of hypothesis is fixed. 

From here, our main result reads as follows.

\begin{theorem}
Given that the minP-property holds, we can declare the pattern $x_i$ significant
(reject the null hypothesis) if the adjusted $p$-value satisfies
$\tilde{p}_{i}^{H} \leq \alpha$ with the guarantee that the FWER 
is controlled at the level $\alpha$.
\label{thm:main}
\end{theorem}

The proof of Theorem \ref{thm:main} is given later in Section \ref{sub:validity}.
The proper definition of minP-property  and a test for checking
whether the calculated $p$-values on the data satisfy this property will
be discussed before in Section~\ref{sect:minp}. 
In practice, we will show in the experiments that in many practical
cases this minP-property is satisfied. 

\section{Theory of multiple testing of data mining results}

\label{sec:theory}

This section validates the result presented in Theorem~\ref{thm:main}.
First, we provide the preliminaries for the multiple hypothesis testing
framework; and next discuss the two empirical $p$-value calculation
methods and the minP-property. Finally, we show the correctness of
the main result of this paper.

In the remainder of the paper we ignore the sampling error due to
the finite number of samples from the null distribution, that is,
we assume that $n$ is large enough. We also assume 
that the data mining algorithm always outputs at least one pattern. 

\subsection{Multiple hypothesis testing}

In this section we provide a short summary of the theory and methods
of multiple hypothesis testing. See \cite{Dudoit03,Westfall93} for a
review and further references.

Consider the problem of testing simultaneously $m$ null hypothesis
$H_{0i}$, $i\in[m]$. It is assumed that the number of hypothesis
to be tested, $m$, is known in advance, while the numbers $m_{0}$
and $m_{1}=m-m_{0}$ of true and false null hypothesis, respectively,
are unknown parameters. With each hypothesis we have associated a
test statistic value $T_{i}$ and a corresponding $p$-value $p_{i}$,
$i\in[m]$. A $p$-value $p_{i}$ is defined as a probability that
the test statistic value is at least $T_{i}$ under the null hypothesis
$H_{0i}$. The values of $p_{i}$s are sometimes called \emph{unadjusted
$p$-values}.

In the simplest case, there is only one hypothesis ($m=1$). A valid
level $\alpha$ statistical test is such that the hypothesis is declared
significant, i.e., the null hypothesis is rejected, if $p_{1}\le\alpha$.
This happens with a probability of at most $\alpha$ if the data is
sampled from the null distribution. Falsely declaring a pattern significant
(false positive) is called a Type I error, while falsely declaring
a pattern non-significant (false negative) is called a Type II error.
A standard approach is to specify an acceptable level $\alpha$ for
the Type I error rate and construct a test, i.e., choose a test statistic,
that minimizes the Type II error rate, that is, maximizes the \emph{power}
of the test.

For multiple hypothesis testing $m>1$, the situation is no longer
as straightforward. Following \cite{Dudoit03}, we denote
by $R$ the number of hypothesis declared significant; by $S$ and
$U$ the numbers of hypothesis correctly declared significant and
non-significant, respectively; and by $V$ and $T$ the number of
hypothesis declared incorrectly significant and non-significant, respectively.
The count $V$ corresponds to the number of Type I errors (false positives),
while $T$ corresponds to the number of Type II errors (false negatives).
See Table \ref{tab:multiple} for a summary. %
\begin{table}
\begin{centering}
\begin{tabular}{l|cc|l}
 & Not declared  & Declared  & \tabularnewline
 & significant  & significant  & \tabularnewline
\hline 
True null hypothesis  & $U$  & $V$  & $m_{0}$\tabularnewline
Non-true null hypothesis  & $T$  & $S$  & $m_{1}$\tabularnewline
\hline 
 & $m-R$ & $R$ & $m$ \tabularnewline
\end{tabular}
\par\end{centering}

\caption{\label{tab:multiple} Multiple hypothesis testing. $R$ and $m$ are
observed counts, while $S$, $T$, $U$, $V$, $m_{0}$ and $m_{1}$
are unknown. $V$ is the number of Type I errors and $T$ the number
of Type II errors.}

\end{table}

There are many ways to define the acceptable Type I error rate. We
use the \emph{family-wise error rate} (FWER). A statistical test
that controls the FWER at level $\alpha$ is such that the probability
of even one Type I error is at most $\alpha$, that is, $Pr(V>0)\le\alpha$.
Another control of Type I error is given by the \emph{false discovery
rate} (FDR), introduced by \cite{Benjamini95}. A statistical test
that controls the FDR at level $\alpha$ is such that the expected
fraction of Type I errors among the rejected hypothesis is at most
$\alpha$, that is, $E(Q)\le\alpha$, where $Q=V/R$ if $R>0$ and
$0$ if $R=0$.

The choice of control depends on the application. If even one false
positive would be disastrous, for example, the hypothesis would be
about if the various drugs are safe to use, then it is appropriate
to choose FWER. However, the FDR may be more appropriate choice, for
example, if the objective is to identify hypothesis for further study.

The multiple hypothesis testing methods are often defined in terms
of adjusted $p$-values. In the following, we review two tests ---
Bonferroni and Holm-Bonferroni --- that can be used to compute the
adjusted $p$-values while controlling the FWER.

The simplest and probably the best known multiple testing method that
controls the FWER is the Bonferroni test. The adjusted $p$-values
are given by \begin{equation}
\tilde{p}_{i}^{B}=\min{\left(1,mp_{i}\right)}.\label{eq:B}\end{equation}
 A hypothesis $i\in[m]$ is declared significant if $\tilde{p}_{i}^{B}\le\alpha$.

Advantages of the Bonferroni test are that it is simple and easy to
understand and implement, and that an adjusted $p$-value of a hypothesis
depends only on the unadjusted $p$-value of the same hypothesis.

We do not use the Bonferroni test, because a more powerful and slightly
more complicated test that controls the FWER was introduced by \cite{Holm79}:
the Holm-Bonferroni test given in Equation (\ref{eq:HB}). Neither
Bonferroni nor Holm-Bonferroni tests make any assumptions on the dependency
structure of the hypothesis. In our application this is an important
property, as the hypothesis (patterns output by the data mining algorithm)
can have strong correlations.

In the presence of $n$ samples from the null distribution we can
use empirical $p$-values, see \cite{North02} for discussion.

\subsection{The  minP-property assumption}
\label{sect:minp}

Before showing the validity of our method, we present first the minP-property
that we require the algorithm to satisfy.
This property guarantees weak control over
the FWER with the absence of false null hypotheses, i.e., under the
complete null hypothesis. 

\begin{definition}[minP-property]
Assume a dataset $D'$ is sampled from the
null distribution $\Pi_{0}$. Then it holds $\forall t\in [0,1]$ that \[
Pr(|A(D')|\min_{x\in A(D')}p_{D'}(x)\le t|H_{0}^{C})\le t,\]
where $p_{D'}(x)$ is the empirical $p$-value of the pattern $x$
output by the algorithm with dataset $D'$, and $H_{0}^{C}$ signifies
the complete null hypothesis.
\label{def:minp}
\end{definition}

Indeed, the minP-property enforces constraint to the way the test statistics
can vary for different number of outputs. If larger test statistics
are encountered only for a small number of output patterns, there
is an elevated risk of false positives. If a pattern has an extreme
test statistic value, it will have a small $p$-value. For a small
number of patterns, the Holm-Bonferroni adjustment will be small.
If both cases are true, the adjusted $p$-value will also be small,
possibly causing a false positive.

As shown later in Lemma~\ref{lem:minpm}, the minP property is always
satisfied if the data mining algorithm always outputs a constant
number of patterns. Our defined empirical $p$-values might not satisfy
the minP-property in all cases when the number of patterns output by
the algorithm varies. We define the following to test if the property
holds.

\begin{definition}[minP-test]
\label{def:minP-test}
Let $\hat{p}_{i}=|A(D_{i})|\min_{x\in A(D_{i})}p_{D_{i}}(x)$.
The minP-property is satisfied if for all $t\in[0,1]$,\[
\frac{|\{i|\hat{p}_{i}\le t\}|}{n}\le t.\]
\end{definition}

Notice that the minP-test can be carried out visually by 
plotting $\frac{|\{i|\hat{p}_{i}\le t\}|}{n}$
against $t$ and checking if the plotted line never exceeds
the diagonal line.

Actually, our two defined $p$-values admit the minP-property in a
variety of situations.  We make the following observation concerning
both methods.

\begin{lemma}
\label{lem:minpm}
The minP-property is always satisfied for both the sample and pool-based
methods if the data mining algorithm $A$ outputs a constant number
of patterns, that is, $m=|A(D')|$ for any $D'$. In this case, the
two $p$-values behave in the same way.\label{lem:constant}
\end{lemma}

For proving Lemma~\ref{lem:constant} we need first the following property.

\begin{proposition}
\label{lem:integral}For real valued $y$ and $x$, that are distributed
identically, and for any $\alpha\in[0,1]$, $Pr(Pr(x \le y)\le\alpha)=\alpha.$
\end{proposition}

\begin{proof}{\bf of Lemma~\ref{lem:constant}.}
The sample-based $p$-values can be written as:
\begin{eqnarray*}
p_{D'}^{sample}(x) & = & \frac{\sum_{i=1}^{n+1}h(x,D',D_{i})}{n+1}\nonumber \\
 & \stackrel{n\rightarrow\infty}{\longrightarrow} & \sum_{\hat{D}}Pr(\hat{D})h(x,D',\hat{D})\label{eq:sampleinf}.\\
\end{eqnarray*}
With this, the minP-property is then
\begin{eqnarray*}
 &  & Pr(m\min_{x\in A(D')}p_{D'}^{sample}(x)\le\alpha|H_{0}^{C})\\
 & = & \sum_{D'}Pr(D')\mathbb{I}\left(m\min_{x\in A(D')}p_{D'}^{sample}(x)\le\alpha\right)\\
 & = & \sum_{D'}Pr(D')\mathbb{I}\left(\min_{x\in A(D')}\left(\sum_{\hat{D}}Pr(\hat{D})\left|\{y\in A(\hat{D})\mid f(x,D')\le f(y,\hat{D})\}\right|\right)\le\alpha\right)\\
 & \le & \sum_{D'}Pr(D')\mathbb{I}\left(\sum_{\hat{D}}Pr(\hat{D})\mathbb{I}\left(\max_{x\in A(D')}f(x,D')\le\max_{y\in A(\hat{D})}f(y,\hat{D})\right)\le\alpha\right)\\
 & = & \alpha.\end{eqnarray*}
The function $\mathbb{I}(\cdot)$ returns 1 if the condition is true and 0 otherwise. The last step follows from Proposition~\ref{lem:integral}. To prove the
equality of the methods, consider\begin{eqnarray*}
p_{D'}^{sample}(x) & = & \frac{\sum_{i=1}^{n+1}h(x,D',D_{i})}{n+1}\\
 & = & \frac{1}{(n+1)m}\sum_{i=1}^{n+1}|\{y\in A(D_{i})|f(y,D_{i})\ge f(x,D')\}|\\
 & = & p_{D'}^{pool}(x).\end{eqnarray*}
\qed
\end{proof}

That is, in the simplest case, where the data mining algorithm is
expected to output approximately constant number of patterns, the
minP-property is expected to hold to a good accuracy. 

The minP-property is in practice not too restrictive, as shown later
by our experiments. A data mining algorithm may violate
the property if the distribution of $p$-values depends strongly on
the number of patterns output by the algorithm.

\begin{example}
\emph{Adversarial example.} Assume that we have a data mining algorithm
$A$ and null distribution such that when a dataset $D$ is sampled
from the null distribution the following is satisfied. With probability
of $\frac{4}{5}$, the algorithm $A(D)$ outputs one pattern with
a sample-based $p$-value $p_{D}^{sample}(x)$ sampled from uniform
$U(\frac{1}{10},1)$; and with probability of $\frac{1}{5}$, the
algorithm $A(D)$ outputs two patterns, one having a $p$-value from
$U(0,\frac{1}{10})$ and another having a $p$-value from $U(\frac{1}{10},1)$.
Here $U(a,b)$ is a probability distribution over real numbers that
is uniform over interval $[a,b]$ and zero elsewhere.

The above described pattern of $p$-values would occur for example
if the algorithm would output with probability $\frac{4}{5}$ a pattern
with a test statistics $f$ from $U(-1,0)$; and with probability
of $\frac{1}{5}$ two patterns with test statistics from $U(-1,0)$
and $U(0,1)$, respectively.

Choose $t=\frac{3}{5}$ and denote $p'=\min_{x\in A(D')}p_{D'}^{sample}(x)$.
If the algorithm outputs one pattern ($m_{0}=1$) we have $Pr(m_{0}p'\le\frac{3}{5})=\frac{5}{9}$
--- this happens with a probability of $\frac{4}{5}$. On the other
hand, if the algorithm outputs two patterns ($m_{0}=2$) we have $Pr(m_{0}p'\le\frac{3}{5})=1$
(as the smallest $p$-value is always at most $\frac{1}{10}$).

Summarizing, $Pr(m_{0}p'\le\frac{3}{5})=\frac{4}{5}\times\frac{5}{9}+\frac{1}{5}\times1=\frac{29}{45}\approx0.644$,
which does not satisfy the minP-property. Furthermore, the minP-property
neither holds for the pool-based $p$-values in this example.
\end{example}

Assessing the minP-property under the combination of randomization
method, algorithm and test statistic may be prohibitively complex
to do analytically. Still Definition~\ref{def:minP-test} corresponds
to a test for the minP-property, which indicates
whether the minP-property is violated. 

In Figure~\ref{fig:ar_minp}, the visual test for minP-property
is illustrated. The plot shows the empirical FWER under the complete
null hypothesis for different acceptance levels for FWER. If the diagonal
line is exceeded, the minP-property does not hold. However, since
the method is approximate, slight violations may be due to sampling
error and may be ignored. 

\begin{figure}[t]
\begin{centering}
\includegraphics{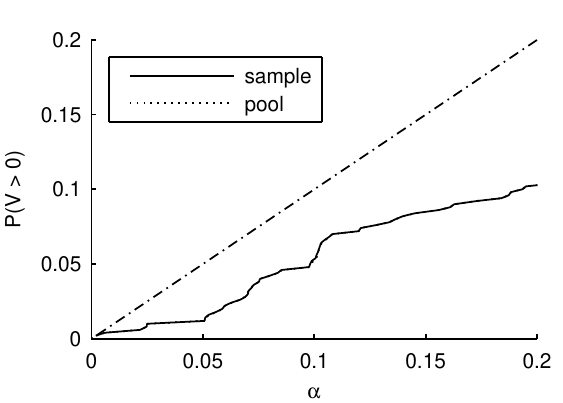}
\par\end{centering}
\caption{Visual minP-test. The horizontal and vertical axes represent
  the acceptable and true FWER levels, respectively. The solid line
  represents the measured FWER for different acceptable FWER levels,
  and the dash dotted line represents the threshold that should not be
  exceeded. We have used association rule mining for all methods with
  $\paleo$ dataset, see Section \ref{sec:assrules}. Both methods
  coincide.\label{fig:ar_minp}}
\end{figure}

\subsection{Validity of our method}

\label{sub:validity}

We prove that our method defined in Theorem~\ref{thm:main},
is a valid level $\alpha$ test by an argument
similar to the closed testing procedure~\cite{Marcus76} with minP
test, which is arguably the most concise way to derive the Holm-Bonferroni
test in the traditional multiple testing scenario. We finally conclude with
the proof of Theorem~\ref{thm:main}.

\begin{proof}{\bf of Theorem~\ref{thm:main}.}
Assume that the data mining algorithm outputs $m$ patterns $A(D)=P_{0}\cup P_{1}$
such that the patterns in $P_{0}$ obey the null hypothesis and the
patterns in $P_{1}$ do not, with $P_{0}\cap P_{1}=\emptyset$. Let
$m_{0}=|P_{0}|$. If $m_{0}=0$ the FWER is always trivially controlled;
in the following we consider the case $m_{0}>0$. Denote by $x'$
the pattern in $P_{0}$ that has the smallest $p$-value, that is,
$x'=\arg{\min_{x\in P_{0}}{p_{D}(x)}}$. In the Holm-Bonferroni test
of Equation (\ref{eq:HB}), we violate the FWER ($V>0$) if and only
if we declare $x'$ significant. Let $m'$ be the number of patterns
with a $p$-value no smaller than $p_{D}(x')$; $m'$ obviously satisfies
$m_{0}\le m'$. In the Holm-Bonferroni test we declare $x'$ significant
if $m'p_{D}(x')\le\alpha$. Due to minP-property,\[
Pr\{m'p_{D}(x')\le\alpha\}\le Pr\left\{ m_{0}p_{D}(x')\le\alpha\big|H_{0}^{C}\right\} \le\alpha.\]
The first inequality holds, since $m_{0}\le m'$ and $P_{1}$ do not
effect $p_{D}(x').$ Therefore,\[
Pr\{x'{\rm ~declared~significant}\}\le\alpha,\]
and in other words, $V>0$ with a probability of at most $\alpha$.
\qed
\end{proof}

\section{Related work}

\label{sec:relatedwork}

In the following section, we review the existing methods from the
literature related to the multiple hypothesis testing within the data
mining framework. Notice that none of these methods is directly comparable
to our contribution. The reason for this is that they control different
error or only calculate it, use specific randomization to derive the
significance, or are defined only for specific types of patterns.

\subsection{Methods measuring Type I errors}

In \cite{Zhang04}, Zhang et al. mine significant statistical quantitative
rules (SQ rules). The difference to our methods is that they only
calculate FDR, where we prove strong control for FWER. Also, they
restrict their approach to a specific setting in binary data and consider
only one class of null hypotheses. Our methods are not restricted to
a special type of pattern and allow a very broad spectrum of null
hypotheses. 

In the method of Zhang et al., the dataset is splitted into two sets
of attributes $X$ and $Y$. The antecedent, i.e. an itemset, of a
SQ rule is defined as a subset of $X$. A statistic value is also
attached to each rule, where the value is calculated from the transactions
that contain the antecedent in $X$, but only using the attributes
in $Y$. The significances are calculated by randomizing the dataset
so that the dependence between antecedent and consequent is broken,
and the statistic values are stored for each randomized dataset and
rule. Finally, a confidence interval for a single rule is defined
as the interval to which $(1-\alpha)$ statistic values fall. FDR is not shown to be statistically controlled, but it is calculated
by using randomized datasets and checking how many rules are declared
significant with a certain $\alpha$. The value of FDR is then obtained
by taking the mean of the numbers and dividing it by the number of
rules discovered from the original data.\\

In \cite{Megiddo98}, Megiddo et al. mine association rules and, given
a threshold $\alpha$ for the raw $p$-values, calculate the expected
number of Type I errors. Conversely, our methods have strong \emph{control
}for FWER. Furthermore, they consider only association rules, and
thus the method might not be as general as our methods.

In the method of Megiddo et al., they start by first mining frequent
itemsets and then using them to find association rules that have a
sufficiently large minimum confidence. The $p$-values are calculated
for itemsets from a Gaussian distribution with the mean set to the
minimum support value used to mine the itemsets, and the variance
set to the variance of a binomial distribution with the probability
$\mathrm{(minimum\ support)}/\mathrm{(nr\ transactions)}$. The authors
also discuss the $p$-values for association rules, but unfortunately
do not explicitly state how to calculate them. The multiple testing
procedure is then to construct random datasets with the same expected
column margins as the original dataset. The $p$-values for all patterns
in a single dataset are sorted ascending, and the mean of the smallest
$p$-values over all datasets is calculated as $V_{1}$, then the
second smallest as $V_{2}$, and so forth for $V_{k}$. These values
define thresholds for $\alpha$-values. For instance, if $V_{1}<\alpha\le V_{2}$,
then the expected number of Type I errors is at most two for the level
$\alpha$.

\subsection{Method that controls the probability of at most $V_{0}$ Type I errors}

A generalization of FWER is to control the probability that the number
of Type I errors exceeds a specified number, $V_{0}.$ In other words,
assure that $Pr(V>V_{0})\le\alpha.$ Standard FWER is controlled when
$V_{0}=0.$ 

The method by Lallich et al.~\cite{Lallich06b} is similar to ours in
that it draws random samples of datasets and defines control over FWER
(or similar) Type I error measure. The difference is that they define
to use bootstrapping, and thus the method might not be as general as
ours as the properties of the method may depend on bootstrapping.
Furthermore, the correction for multiple hypothesis is calculated
directly, which may require assumptions~\cite{Westfall93} and strong
control is not proven, which we do.

In the method, they first find a set of association rules from the
original data, and calculate some statistic for each rule. Then they
sample random datasets by using bootstrap over the transactions with
replacement. For each rule, the same statistics are calculated in
the random dataset, and the difference in the statistic values between
the original and random data are computed. The differences are sorted
in decreasing order and stored as $\epsilon(k,i)$, where $i$ is
the index of a random dataset and $k$ is the rank of a difference.
Finally, for a certain desired number $V_{0}$, a value $\epsilon(\alpha)$
is calculated that satisfies \[
\frac{|\{i|\epsilon(V_{0},i)\ge\epsilon(\alpha)\}|}{\mathrm{nr\ random\ datasets}}\le\alpha.\]
 The rules that have a statistic value higher than $\epsilon(\alpha)$
are selected.

\subsection{Methods that control the FWER}

In \cite{Bay01}, Bay and Pazzani mine contrast sets. The similarly
to our methods is the control of FWER, but they restrict themselves
to contrast sets, where our methods are general. Furthermore, using
Bonferroni correction may not be reasonable, since it is often overly
conservative and provides no theoretical improvement over Holm-Bonferroni. 

Contrast sets are similar to association rules but differ in that
a good rule shows contrast between two groups of transactions. The
data can be multivariate, but it is required that the transactions
are grouped to disjoint sets. The $p$-values are calculated for contrast
sets from a respective contingency table using $\chi^{2}$ approximation.
The rules that have too small values in the contingency table for
the $\chi^{2}$ to produce an adequate approximation are pruned away.
Using the Bonferroni inequality, the authors define confidence thresholds
for each size of itemset mined, which is dependent upon the number
of candidates of a specific size generated when mining frequent itemsets,
\[
\alpha_{l}=\min(\frac{\alpha}{2^{l}|C_{l}|},\alpha_{l-1}),\]
 where $l$ is the level, or size of itemsets, and $|C_{l}|$ is the
size of the candidate set for level $l$.

In \cite{Webb06,Webb07}, Webb considers association rules and
defines a method to overcome the problems caused by thresholding,
with, for example, minimum support. The number of actual hypotheses
may not correspond to the number of output patterns. Webb's main contribution,
the holdout method, is similar to our methods in that it considers
the problematic scenario of varying number of outputs. The Holm-Bonferroni
correction can be and is used in the paper, as do we. However, the
method by Webb is limited to scenarios where the data can be split
into two independent parts, and there is enough data to split it.
Only association rules are considered in the paper. Splitting may
not be possible for example with network data or spatial data. Furthermore,
the data mining algorithm needs to be able to operate on partial data.
Our methods do not have such constraints.

In the paper, Webb presents two methods by using the contingency
table of an association rule to find out its $p$-value. The first
method is to use normal Bonferroni-adjustment for the original $p$-values,
where the multiplier is the number of all possible patterns of at
most a preset maximum length set by the user. The other method is
the holdout method. The data is splitted in two, and a part of the
data, called exploratory data, is used to find the set of itemsets
to consider using normal association rule mining methods. After that,
the second part of the data, called holdout data, is used to assess
the statistical significance of the set of rules.

\subsection{Standard methods for FWER}

The standard methods for adjusting raw $p$-values to control FWER
include, among others, Bonferroni, Holm-Bonferroni, and
Sidak~\cite{Dudoit03}, as well as resampling based methods of
\cite{Westfall93}. A common property of all of these methods is that
they assume that the set of hypothesis (in our case, set of all
possible patterns) are defined beforehand and there is a raw $p$-value
for each of the potential patterns. This poses the problems explained
in the introduction.

\section{Experiments}

\global\long\def\sample{\emph{sample}}
\global\long\def\pool{\emph{pool}}
\global\long\def\minp{\emph{minp}}

\label{sec:experiments}

In this section we show the tests carried out to assess the quality
of the proposed methods.

\subsection{Synthetic data}

The first experiment was with synthetic data of real numbers, with
which the performance of the methods can be measured in a controlled
environment. The synthetic data follows the multivariate Gaussian
distribution \[
\Pi_{0}(x)=\mathcal{N}(x,\mu,C),\]
 where $\mu$ is the mean vector and $C$ the covariance matrix. The
generated real values correspond to the test statistic values $f$
used.

We began by using the methods under the complete null hypothesis and
measuring the empirical probability of rejecting at least one hypothesis,
or $Pr(V>0|H_0^C)$. This corresponds to the minP-test. 

The randomized datasets were generated by drawing a random vector
of length $k=100$ from the normal distribution with zero mean $\mu=0$
and covariance matrix \[
(C_{0})_{ij}=\begin{cases}
1 & ,i=j\\
\sigma & ,i\neq j\end{cases}.\]
The values in a random vector constitute a dataset $D$, which is
also of size $k$. The parameter $\sigma\in[-\frac{1}{k-1},1]$ controls
the amount of covariance between the data points. If $\sigma<-\frac{1}{k-1},$
the covariance matrix is no longer positive semi-definite, and therefore,
no longer a proper covariance matrix. We used $\sigma\in\{-0.0099,0,0.1,0.25,0.5,0.99\}.$

To simulate data mining methods, we used three different ad-hoc algorithms.
These are $\algge$, $\algmax$, and $\algrnd$. Assume now that we
have a dataset $D$ of real numbers. The first algorithm, $\algge$,
outputs the set of values that are greater or equal to 1 in $D$.
The second algorithm, $\algmax$, outputs the 10 largest values in
$D$. And the third algorithm, $\algrnd$, selects 10 numbers from
$D$ uniformly at random.

A single run starts by generating a dataset $D$ from the null distribution.
This data is then mined for patterns $P$ with a selected algorithm,
which results in a set of real values. These values are the test statistics
values of the mined patterns. Then, we draw $n=10000$ datasets from
$\Pi_{0}$, and calculate $p$-values for all $P$ using both methods.
Finally, the minimum value of the adjusted $p$-values are stored,
which is $\min(|P|\min_{y\in P}(p_{D}(y)),1),$ separate
for both methods.

We performed these runs 10000 times for each combination of algorithm
and magnitude of covariance. Figure~\ref{fig:toycorr} depicts the
results: the solid lines correspond to different magnitudes of covariance
and represent for each value of $\alpha$ the empirical probability
of $P(V>0).$

\begin{figure}[ht!]
\begin{centering}
\subfloat[$\algge$ with $\sample$]{\begin{centering}
\includegraphics{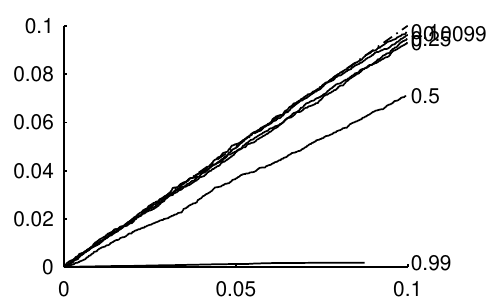}
\par\end{centering}

}\subfloat[$\algge$ with $\pool$]{\begin{centering}
\includegraphics{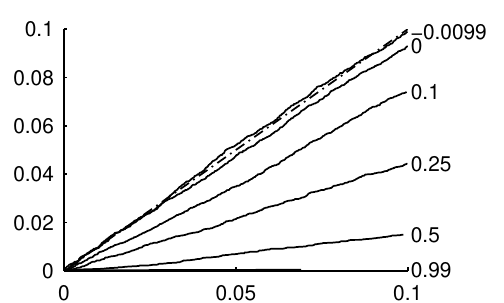}
\par\end{centering}

}\\
\subfloat[$\algrnd$ with $\sample$]{\begin{centering}
\includegraphics{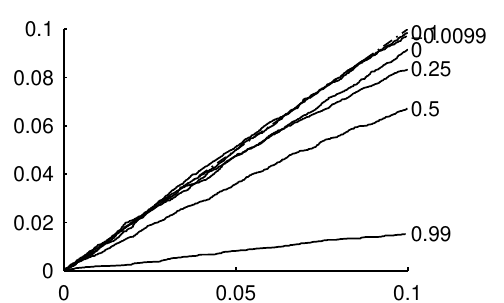}
\par\end{centering}

}\subfloat[$\algrnd$ with $\pool$]{\begin{centering}
\includegraphics{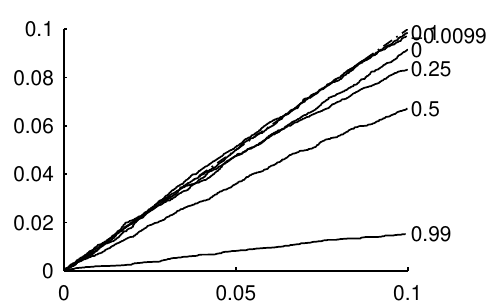}
\par\end{centering}

}\\
\subfloat[$\algmax$ with $\sample$]{\begin{centering}
\includegraphics{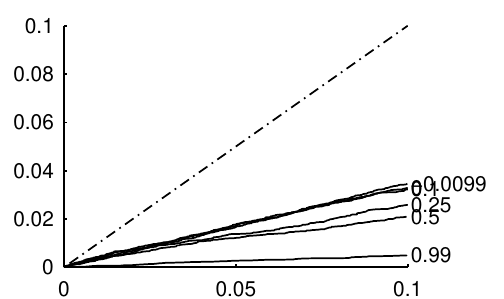}
\par\end{centering}

}\subfloat[$\algmax$ with $\pool$]{\begin{centering}
\includegraphics{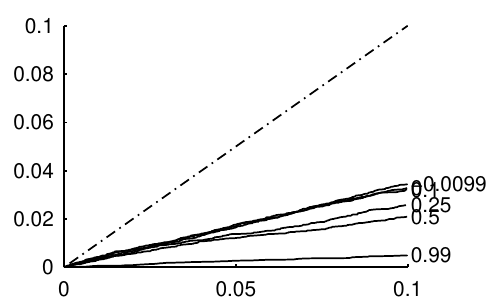}
\par\end{centering}

}
\par\end{centering}

\caption{MinP-test results for synthetic data. The solid lines correspond to
empirical probabilities of Type I error with a specific $\alpha$
level. To the right of line is the magnitude of covariance of the
respective test. The dash-dotted line is the $Pr(V>0)=\alpha$ line,
which should not be exceeded. }

\label{fig:toycorr} 
\end{figure}

As shown in the figure, when the correlation is negative, $-0.0099$,
and algorithm $\algge$ or $\algrnd$ is used, $Pr(V>0)\approx\alpha$.
This means that the proposed methods control the FWER very tightly
in some cases. The important observation is that the controlled threshold
is not exceeded, which translates to satisfying the minP-property.%
\footnote{The threshold is actually slightly exceeded at some points, but this
is due to the finite number of samples $n$ from the null distribution.%
}

The power was also tested with synthetic data of the same kind. Each
dataset was constructed by randomizing samples from the multivariate
Gaussian distribution with mean 0 for samples from the null distribution,
and mean 4 for samples from the alternate distribution, and correlation
between all samples. Hence, \[
(\mu)_{i}=\begin{cases}
0 & ,i\le m_{0}\\
4 & ,m_{0}<i\le m\end{cases},\]
 and \[
(C_{0})_{ij}=\begin{cases}
1 & ,i=j\\
\sigma & ,i\ne j.\end{cases}\]

The number of null hypothesis was set to $m_{0}=80$. The same simulations
with 10000 randomizations for datasets and 10000 overall runs, were
performed for different correlations $\sigma\in\{-0.0099,0,0.1,0.25,0.5,0.99\}$
and the algorithm $\algge$. The probability of Type I error (FWER)
and the mean fraction of Type II errors were calculated for both $p$-value
calculation methods. Figure~\ref{fig:toyroc} depicts ROC-curves
for $\sigma=0.5$ correlation with varying $\alpha.$

\begin{figure}[h]
\begin{centering}
\includegraphics{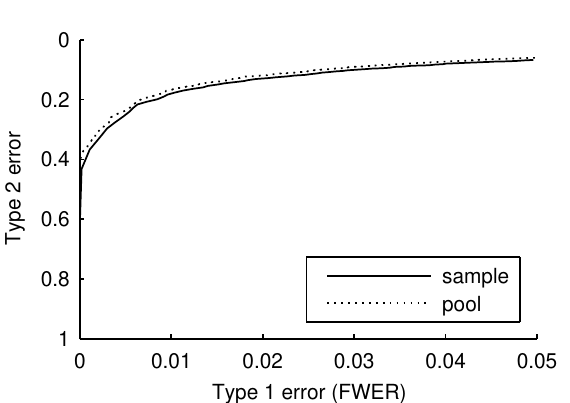} 
\par\end{centering}

\caption{Power results for synthetic data. ROC-curve (varying $\alpha$) of
the $p$-value calculation methods for $\sigma=0.5$ correlation.
\label{fig:toyroc} }

\end{figure}
To compare the methods for all correlations, we also calculated the
area under curve (AUC) from the ROC-curves for both methods and all
correlations. The results are shown in Figure~\ref{fig:toyauc}.%
\begin{figure}[h]
\begin{centering}
\includegraphics{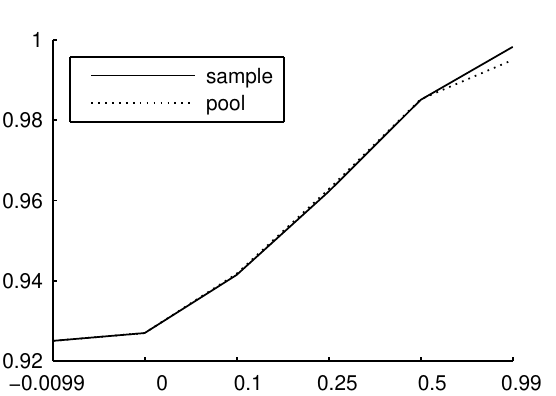} 
\par\end{centering}

\caption{Power results for synthetic data. Area under curves of the ROC-curves
for different $p$-value calculation methods and correlations. Higher
value represents better accuracy. \label{fig:toyauc} }

\end{figure}

As a conjecture from the synthetic data experiments, the minP-property
was always satisfied and the power of both methods are very similar
in these cases.

\subsection{Association rules}
\label{sec:assrules}

The second experiment was a more practical data mining scenario, namely,
association rule mining.

We used three different datasets: Paleo, Courses and Retail; all of
which were used in~\cite{Gionis06}. The property values of these
datasets are presented in Table~\ref{tab:datasets}.

\begin{table}[h]
\begin{centering}
\begin{tabular}{lcccc}
\hline 
Dataset  & \# of rows  & \# of cols  & \# of 1's  & density \%\tabularnewline
\hline 
$\paleo$  & 124  & 139  & 1978  & 11.48\tabularnewline
$\courses$  & 2405  & 5021  & 65152  & 0.54\tabularnewline
$\retail$  & 88162  & 16470  & 908576  & 0.06\tabularnewline
\hline
\end{tabular}
\par\end{centering}

\caption{Description of the datasets\label{tab:datasets}}

\end{table}

Each dataset was randomized 1000 times by maintaining the column margins
constant. Association rules were then mined from each dataset using
the minimum support thresholds from~\cite{Gionis06}: $7$, $400$
and $200$ for $\paleo,$ $\courses$ and $\retail$, respectively.
We used as test statistic $f$ the Fisher's exact test between the
antecedent and the consequent of an association rule. Holm-Bonferroni's
method was used to correct the $p$-values. Table~\ref{tab:ar_stats}
lists for different datasets the minimum support, number of rules
in the original data, and the mean and standard deviation of the number
of rules in the randomized datasets.

\begin{table}
\begin{centering}
\begin{tabular}{lccc}
\hline 
Dataset  & minimum support & $|P|$  & $|P_{i}|$ \tabularnewline
\hline 
$\paleo$ & 7 & 9004 & 577.2(24.0)\tabularnewline
$\courses$ & 400 & 51118 & 379.4(7.4)\tabularnewline
$\retail$ & 200 & 4148 & 2703.9(16.1)\tabularnewline
\hline
\end{tabular}
\par\end{centering}

\caption{Mining parameters and statistics for association rule mining. $|P|$
represents the number of association rules in the original data, and
$|P_{i}|$ is the mean number of association rules with random data.
Standard deviations are shown in parenthesis.\label{tab:ar_stats}}

\end{table}

While randomizing, we carried out the minP-test for all combinations
of dataset and $p$-value calculation method. The first half, 500,
of randomizations were used to gather minimum $p$-values from the
random datasets. By construction, the minimum $p$-values will necessarily
correspond to the largest test statistic values, and therefore, for
the first 500 random datasets, the largest test statistic value was
stored from each. These were then calculated $p$-values using the
latter 500 random datasets and both methods. 

In all cases, the minP-property was clearly satisfied. Figure~\ref{fig:ar_minp}
depicts one test result; all other minP-test results are presented
in Appendix~\ref{sub:minp_ar}.%

We also calculated the number of patterns found significant for different
controlled FWER levels $\alpha$. These results are depicted in Figure~\ref{fig:ar_results}.
The results indicate, that $\sample$ is more powerful than $\pool$
in these cases. This is mostly due to the different $p$-value calculation
methods, but can in part be because of the relative large number of
patterns and limited number of randomizations. 

For $\sample,$ the results are intuitive: When $\alpha$ level, i.e.,
the accepted probability of making a Type I error, increases, the
number of significant patterns increases. To conclude, the results
are reasonable.

\begin{figure}
\begin{centering}
\subfloat[$\paleo$]{\begin{centering}
\includegraphics{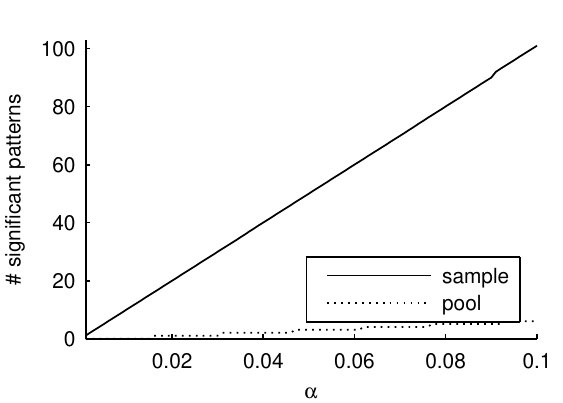}
\par\end{centering}

}\subfloat[$\courses$ ]{\begin{centering}
\includegraphics{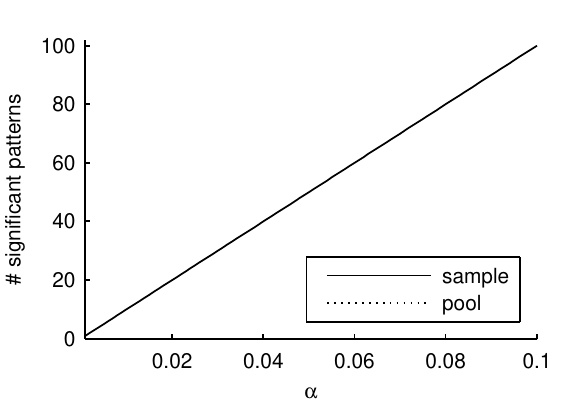}
\par\end{centering}

}\\
\subfloat[$\retail$]{\begin{centering}
\includegraphics{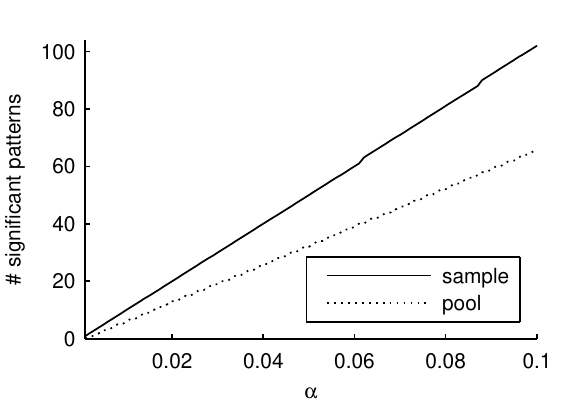}
\par\end{centering}

}
\par\end{centering}

\caption{Association rule mining results for both $p$-value calculation methods
and datasets. The lines depict the number of patterns found significant
for different controlled FWER levels ($\alpha$).\label{fig:ar_results}}

\end{figure}

\subsection{Frequent itemsets}

The third experiment is in similar context to the previous one. However,
this time we mined for frequent itemsets.

The test statistic $f$ was a variant of the lift~\cite{Webb06}:
\begin{equation}
f(x)=\frac{\mathrm{freq}(x)}{\prod_{A\in x}\mathrm{freq}(A)},\label{eq:lift}\end{equation}
where $x$ is an itemset, $A$ is a single attribute, and $\mathrm{freq}(x)\in[0,1]$
is the relative frequency of itemset $x$. The same datasets were
used as in association rule mining with the same minimum support thresholds.
Additionally, we set the smallest frequent itemset size to 2, not
to get individual columns as frequent itemsets. We used the same randomization
method as above that preserves the column margins. We use the name
\textsc{Col} for this method. The second randomization method we used
was presented in~\cite{Gionis06} with the name swap randomization
which additionally maintains the row margins. We will use the name
\textsc{Swap} for this method. Note that we used exactly the same
datasets, mining parameters and randomization methods as in~\cite{Gionis06}.

Each dataset was randomized 10000 times with both methods. Table~\ref{tab:fs_stats}
lists for different datasets the minimum support, number of frequent
itemsets in the original data, and the mean and standard deviation
of the number of frequent itemsets in the randomized datasets. Note
first that the expected numbers of frequent itemsets, and their standard
deviations, are close to the numbers reported in~\cite{Gionis06}.
The small differences are a result of different randomizations.%
\footnote{Additionally, we perform 10000 randomizations while \cite{Gionis06}
do only 500.%
}

\begin{table}
\begin{centering}
\begin{tabular}{lcccc}
\hline 
Dataset  & minsup  & $|P|$  & $|P_{i}^{\mathrm{\textsc{Col}}}|$  & $|P_{i}^{\mathrm{\textsc{Swap}}}|$ \tabularnewline
\hline 
$\paleo$  & 7  & 2828  & 227.4(11.6)  & 266.9(14.8) \tabularnewline
$\courses$  & 400  & 9678  & 146.6(2.8)  & 430.1(11.6) \tabularnewline
$\retail$  & 200  & 1384  & 860.3(7.0)  & 1615.1(11.9) \tabularnewline
\hline
\end{tabular}
\par\end{centering}

\caption{Mining parameters and statistics for frequent itemset maining. $|P|$
is the number of frequent itemsets with the original data; $|P_{i}^{\mathrm{\textsc{Col}}}|$
the mean number of frequent itemsets with random data from \textsc{Col};
and $|P_{i}^{\mathrm{\textsc{Swap}}}|$ the mean number of frequent
itemsets with random data from \textsc{Swap}. Standard deviations
are shown in parenthesis.\label{tab:fs_stats}}

\end{table}

We first carried out the minP-test for all combinations of dataset,
$p$-value calculation and randomization method. In all cases, the
minP-property was satisfied with sufficient accuracy. Figure~\ref{fig:fs_minp}
depicts one test result; all other minP-test results are presented
in Appendix~\ref{sub:minp_fs}.
\begin{figure}[ht!]
\begin{center}
\includegraphics{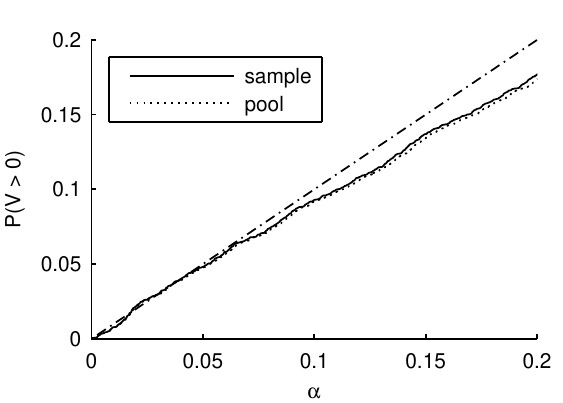} 
\par
\end{center}

\caption{The minP-test in frequent itemset mining for both methods with $\paleo$
dataset and \textsc{Swap} randomization method. The empirical FWER
is plotted against the controlled level. The diagonal dash-dotted
line should not be exceeded for the minP-property to be satisfied.
\label{fig:fs_minp}}

\end{figure}

Figure~\ref{fig:fs_results} depicts the number of patterns found
significant for different $\alpha$ levels for both randomization
and $p$-value calculation methods, and all datasets. From the results
it is clear the $\sample$ method is more powerful in all but $\retail$
with \textsc{Swap}. The reasons for this difference are the same as
in association rule mining. Note also that the swap randomization
is more restricted and, as expected, less patterns were found significant
in comparison to the other randomization approach.

\begin{figure}[ht!]
\begin{centering}
\subfloat[$\paleo$ with \textsc{Col} ]{\begin{centering}
\includegraphics{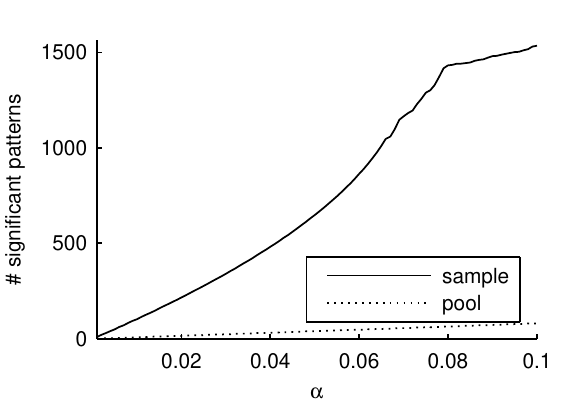}
\par\end{centering}
}\subfloat[$\paleo$ with \textsc{Swap}]{\begin{centering}
\includegraphics{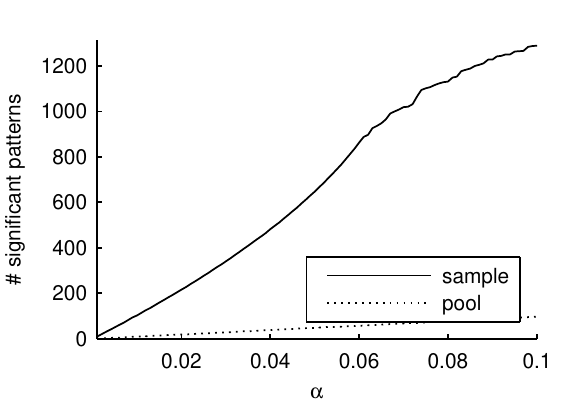}
\par\end{centering}}\\
\subfloat[$\courses$ with \textsc{Col} ]{\begin{centering}
\includegraphics{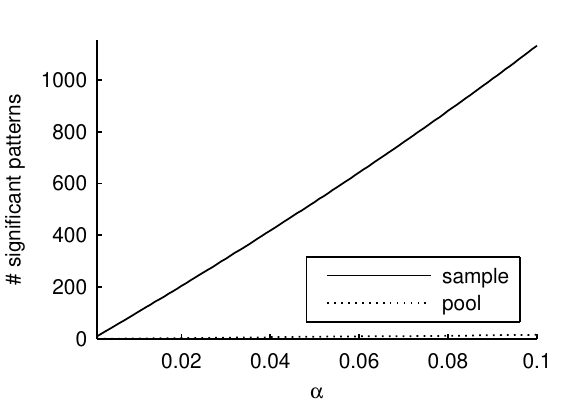}
\par\end{centering}
}\subfloat[$\courses$ with \textsc{Swap}]{\begin{centering}
\includegraphics{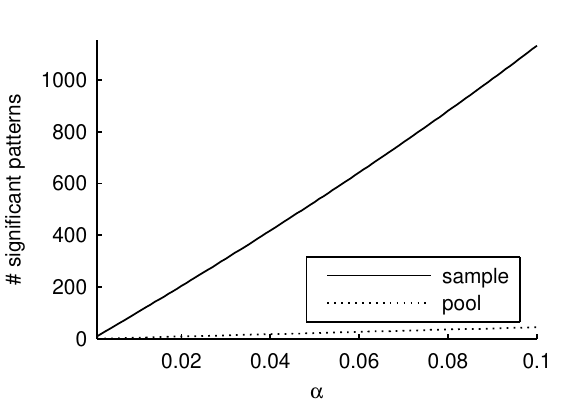}
\par\end{centering}}\\
\subfloat[$\retail$ with \textsc{Col} ]{\begin{centering}
\includegraphics{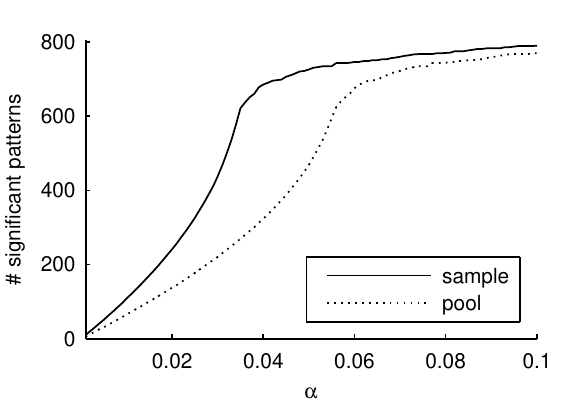}
\par\end{centering}
}\subfloat[$\retail$ with \textsc{Swap}]{\begin{centering}
\includegraphics{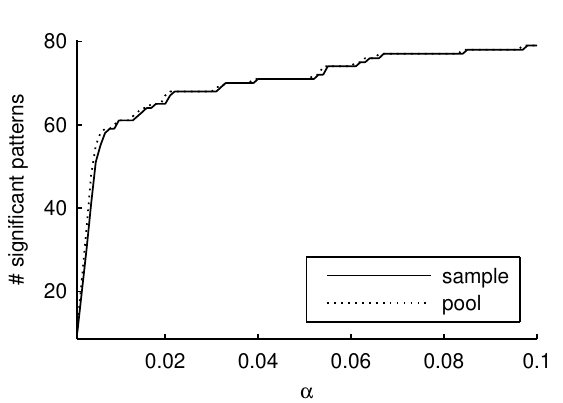}
\par\end{centering}}
\par\end{centering}

\caption{Frequent set mining results for both $p$-value calculation and randomization
methods, and datasets. The lines depict the number of patterns found
significant for different controlled FWER levels ($\alpha$).\label{fig:fs_results}}

\end{figure}

\subsection{Frequent subgraphs}

As a final experiment, we show how the methods can be used in the
setting of frequent subgraph mining. The problem is very similar to
finding frequent itemsets, but now the transactions are graphs and a
frequent pattern is a subgraph of the input graphs. We used
FSG~\cite{Kuramochi04} as a graph mining algorithm, which is a part of
Pafi%
\footnote{http://glaros.dtc.umn.edu/gkhome/pafi/overview%
} and readily available at the website of Karypis Laboratory. As a
dataset, we used a graph transaction dataset of different compounds%
\footnote{http://www.doc.ic.ac.uk/\textasciitilde{}shm/Software/Datasets/\\
  carcinogenesis/progol/carcinogenesis.tar.Z%
}, which has 340 different graphs and the largest graph has 214 nodes.
We calculated the test statistic $f$ for each subgraph as \[
f(x)=\mathrm{freq}(x)\log(\mathrm{\#\, nodes\, in}\, x).\] The
logarithm term is to weight larger subgraphs slightly more, because
they are considered more interesting than small ones.

We randomized the graphs by selecting two edges and switching the
end points together, mixing the edges between nodes. If switching
edges would create overlapping edges, the swap is not performed. The
method preserves the node degrees while creating a completely different
topology for the graph. This randomization has been used before in~\cite{Sharan05}
and later extended in~\cite{Hanhijarvi09}. Since our dataset is
a set of graphs, we randomized each graph individually by attempting
500 swaps, and combined the randomized graphs back to a transactional
dataset.%
\footnote{Notice that the test statistic may be unjustified in the chemistry
domain, and the randomization method may violate some laws of physics.
Despite this, we use them here to show that the methods can be used in
this setting as well.%
}

We used 10000 random datasets at support level 40, and calculated
the $p$-values with both methods. Statistics of the randomizations
are depicted in Table~\ref{tbl:fg}.

\begin{table}
\begin{centering}
\begin{tabular}{ccc}
\hline 
minsup  & $|P|$  & $|P_{i}|$ \tabularnewline
\hline 
40  & 140  & 191.8(13.4) \tabularnewline
\hline
\end{tabular}
\par\end{centering}

\caption{Frequent subgraph mining.$|P|$ is the number of frequent
subgraphs with the original data; $|P_{i}|$ the expected number of
frequent graphs with random data. Standard deviation is shown in parenthesis.\label{tbl:fg}}

\end{table}
The minP-test and the number of patterns found significant for different
$\alpha$ levels are shown in Figure~\ref{fig:fsg-results}. The
minP-property is satisfied, and the power of both methods are similar.
As a conclusion, the $p$-value calculation methods can also be used
in frequent subgraph mining.

\begin{figure}
\subfloat[minP test]{\begin{centering}
\includegraphics{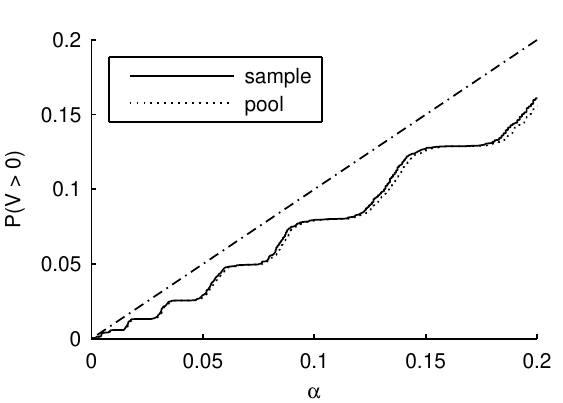}
\par\end{centering}

}\subfloat[power]{\begin{centering}
\includegraphics{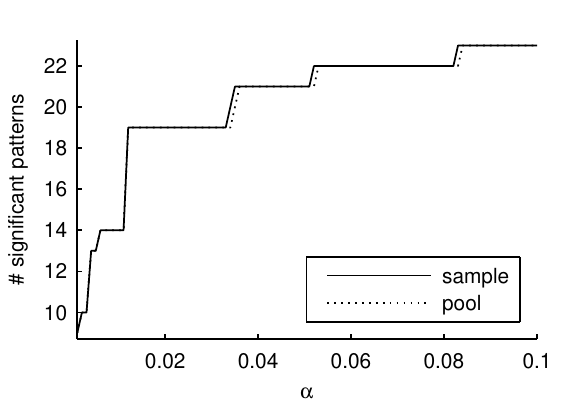}
\par\end{centering}

}

\caption{Frequent subgraph mining results with \textsc{Compound} dataset. In
a), the true FWER is plotted against the controlled level. The diagonal
dash-dotted line should not be exceeded for the minP-property to be
satisfied. In b), number of patterns found significant for different
controlled FWER levels.\label{fig:fsg-results}}

\end{figure}

\section{Discussion and conclusions}

\label{sec:conclusion}

As shown by the recent interest in randomization methods, there is
a clear need for new significance testing methods in data mining applications.
Especially within the framework of multiple hypothesis testing, the
significance tests for data mining results have been lacking.

In this paper, we have introduced two methods to test the significance
of patterns found by a generic data mining algorithm. Unlike much
of the previous work, we do make only very general assumptions of
the data mining algorithm and no assumptions at all of the data nor
on the dependency structure of the patterns output by the data mining
algorithm. Hence, our approach is suitable for many, if not most,
data mining scenarios.

The only assumption we need to make for the purposes of the proof is
that the algorithm satisfies the minP-property of Definition
\ref{def:minp}.  It is possible to find adversarial examples of data
mining algorithms that fail to satisfy the minP-property; however, our
results with toy and real data show that our methods behave
consistently and hence we argue that this is not a serious limitation
in practice. In any case, having such an assumption is not
extraordinary in significance testing. Most of the existing
significance testing methods in fact make some simplifying assumptions
of the distribution of the test statistics; these methods are
conventionally considered reliable if the assumptions are at least
approximately satisfied.

In the paper, we have studied and the scenario where the FWER is being
controlled as our proof of Theorem \ref{thm:main} is specific to
the Holm-Bonferroni test that controls the FWER. However, in many
cases the control of FDR could be a better choice --- for example
in exploratory data analysis where we are looking for patterns that
would warrant a more detailed study. Intuitively, replacing the Holm-Bonferroni
test of Equation (\ref{eq:HB}) with a test that controls the FDR,
such as Benjamini-Yekutieli \cite{Benjamini01}, should work; the
proof of this conjecture is however left for future work.

On real-world datasets, our experiments show that the proposed methods
are also powerful. Hence,
we not only control the FWER under the desired $\alpha$ level, but
also the method avoids as much as possible the false negatives. This
is related to the fact that due to the nature of randomization we
can choose the null hypothesis very freely.


\begin{thebibliography}{10}

\bibitem{Bay01}
Stephen~D. Bay and Michael~J. Pazzani.
\newblock Detecting group differences: Mining contrast sets.
\newblock {\em Data Min. Knowl. Discov.}, 5(3):213--246, 2001.

\bibitem{Benjamini95}
Yoav Benjamini and Yosef Hochberg.
\newblock Controlling the false discovery rate: A practical and powerful
  approach to multiple testing.
\newblock {\em Journal of the Royal Statistical Society. Series B
  (Methodological)}, 57(1):289--300, 1995.

\bibitem{Benjamini01}
Yoav Benjamini and Daniel Yekutieli.
\newblock The control of the false discovery rate in multiple testing under
  dependency.
\newblock {\em The Annals of Statistics}, 29(4):1165--1188, 2001.

\bibitem{Dudoit03}
Sandrine Dudoit, Juliet~Popper Shaffer, and Jennifer~C. Boldrick.
\newblock Multiple hypothesis testing in microarray experiments.
\newblock {\em Statistical Science}, 18(1):71--103, 2003.

\bibitem{Gionis06}
Aristides Gionis, Heikki Mannila, Taneli Mielik\"ainen, and Panayiotis
  Tsaparas.
\newblock Assessing data mining results via swap randomization.
\newblock In {\em Proceedings of the 12th ACM Conference on Knowledge Discovery
  and Data Mining (KDD)}, 2006.

\bibitem{Hanhijarvi09}
Sami Hanhij\"arvi, Gemma~C. Garriga, and Kai Puolam\"aki.
\newblock Randomization techniques for graphs.
\newblock In {\em Proceedings of the 2009 SIAM International Conference on Data
  Mining (SDM 09)}, 2009.

\bibitem{Holm79}
S.~Holm.
\newblock A simple sequentially rejective multiple test procedure.
\newblock {\em Scandinavian Journal of Statistics}, 6:65--70, 1979.

\bibitem{Kuramochi04}
Michihiro Kuramochi and George Karypis.
\newblock An efficient algorithm for discovering frequent subgraphs.
\newblock {\em IEEE Trans. Knowl. Data Eng.}, 16(9):1038--1051, 2004.

\bibitem{Lallich06}
Stéphane Lallich, Olivier Teytaud, and Elie Prudhomme.
\newblock Association rule interestingness: measure and statistical validation.
\newblock {\em Quality Measures in Data Mining}, pages 251--275, 2006.

\bibitem{Lallich06b}
Stéphane Lallich, Olivier Teytaud, and Elie Prudhomme.
\newblock Statistical inference and data mining: false discoveries control.
\newblock In {\em 17th COMPSTAT Symposium of the IASC, La Sapienza, Rome},
  pages 325--336, 2006.

\bibitem{Lehmann59}
E.~L. Lehmann.
\newblock {\em Testing Statistical Hypotheses}.
\newblock Wiley, 1956.

\bibitem{Marcus76}
Ruth Marcus, Eric Peritz, and K.~R. Gabriel.
\newblock On closed testing procedures with special reference to ordered
  analysis of variance.
\newblock {\em Biometrica}, 63(3):655--660, 1976.

\bibitem{Megiddo98}
Nimrod Megiddo and Ramakrishnan Srikant.
\newblock Discovering predictive association rules.
\newblock In {\em Knowledge Discovery and Data Mining}, pages 274--278, 1998.

\bibitem{North02}
B.~V. North, D.~Curtis, and P.~C. Sham.
\newblock A note on the calculation of empirical {P} values from {M}onte
  {C}arlo procedures.
\newblock {\em The American Journal of Human Genetics}, 71(2):439--441, 2002.

\bibitem{Ojala08sdm}
Markus Ojala, Niko Vuokko, Aleksi Kallio, Niina Haiminen, and Heikki Mannila.
\newblock Randomization of real-valued matrices for assessing the significance
  of data mining results.
\newblock In {\em Proceedings of the 2008 SIAM International Conference on Data
  Mining}, pages 494--505, 2008.

\bibitem{Sharan05}
Roded Sharan, Trey Ideker, Brian Kelley, Ron Shamir, and Richard~M. Karp.
\newblock Identification of protein complexes by comparative analysis of yeast
  and bacterial protein interaction data.
\newblock {\em Journal of Computational Biology}, 12(6):835--846, 2005.

\bibitem{Webb06}
Geoffrey~I. Webb.
\newblock Discovering significant rules.
\newblock In {\em KDD '06: Proceedings of the 12th ACM SIGKDD international
  conference on Knowledge discovery and data mining}, pages 434--443, New York,
  NY, USA, 2006. ACM.

\bibitem{Webb07}
Geoffrey~I. Webb.
\newblock Discovering significant patterns.
\newblock {\em Mach. Learn.}, 68(1):1--33, 2007.

\bibitem{Westfall93}
Peter~H. Westfall and S.~Stanley Young.
\newblock {\em Resampling-based multiple testing: examples and methods for
  p-value adjustment}.
\newblock Wiley, 1993.

\bibitem{Zhang04}
Hong Zhang, Balaji Padmanabhan, and Alexander Tuzhilin.
\newblock On the discovery of significant statistical quantitative rules.
\newblock In {\em KDD '04: Proceedings of the tenth ACM SIGKDD international
  conference on Knowledge discovery and data mining}, pages 374--383, New York,
  NY, USA, 2004. ACM.

\end{thebibliography}

\appendix

\section{Extended results}

\subsection{MinP-tests for association rule mining}
\label{sub:minp_ar}%

Figure \ref{fig:99} shows minP test for different $p$-value
calculation methods and datasets.

\begin{figure}[h]
\centering{}\subfloat[\textsc{$\paleo$}]{\includegraphics{paleo_7_sar_1000_col_minp}

}\subfloat[\textsc{$\courses$}]{\includegraphics{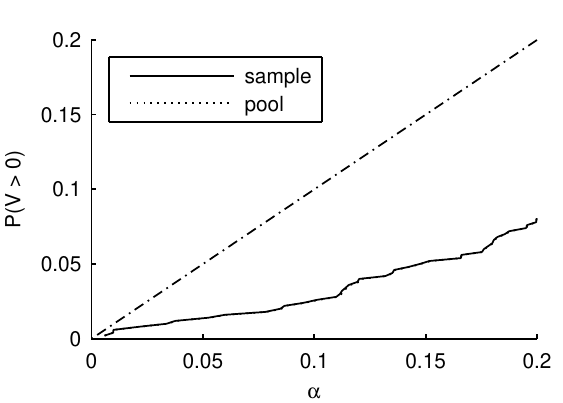}

}\\
\subfloat[\textsc{$\retail$}]{\includegraphics{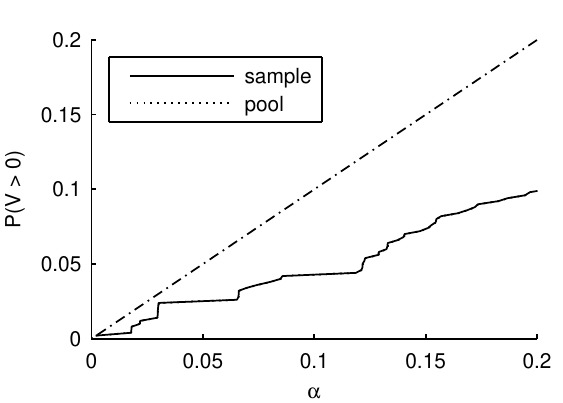}

}\caption{\label{fig:99}
minP test for different $p$-value calculation methods and datasets.
The true FWER is plotted against the controlled level. The diagonal
dash-dotted line should not be exceeded for the minP-property to be
satisfied. }

\end{figure}

\subsection{MinP-tests for frequent itemset mining}
\label{sub:minp_fs}%

Figures \ref{fig:99pal}--\ref{fig:99retail} show the minP test for
different $p$-value calculation and randomization methods.

\begin{figure}[h]
\subfloat[\textsc{Col}]{\includegraphics{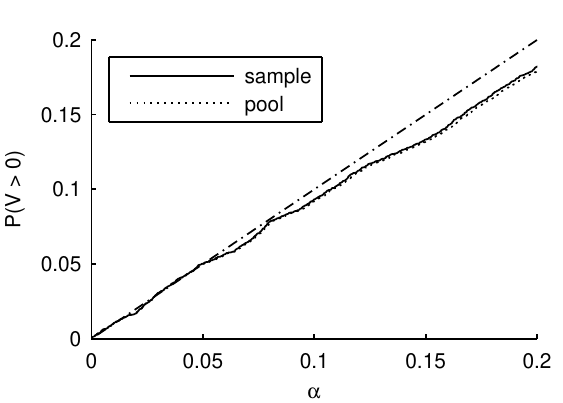}

}\subfloat[\textsc{Swap}]{\includegraphics{paleo_7_10000_swap_lift_minp}

}\caption{\label{fig:99pal}
minP test for different $p$-value calculation and randomization methods
with $\paleo$. The true FWER is plotted against the controlled level.
The diagonal dash-dotted line should not be exceeded for the minP-property
to be satisfied. }

\end{figure}[h]
\begin{figure}
\subfloat[\textsc{Col}]{\includegraphics{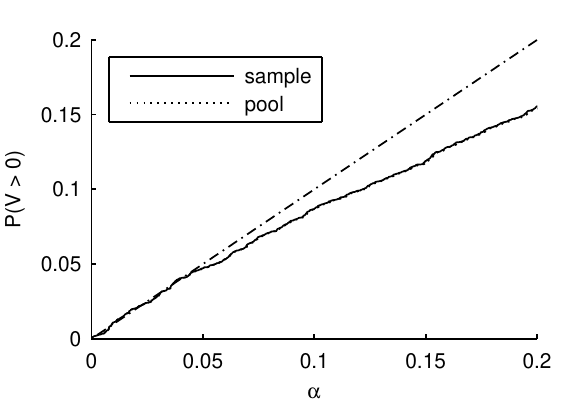}

}\subfloat[\textsc{Swap}]{\includegraphics{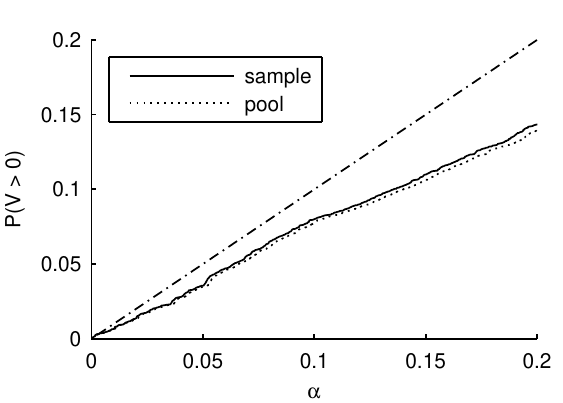}}\caption{\label{fig:99cou}
minP test for different $p$-value calculation and randomization methods
with $\courses$. The true FWER is plotted against the controlled
level. The diagonal dash-dotted line should not be exceeded for the
minP-property to be satisfied. }

\end{figure}
\begin{figure}[h]
\subfloat[\textsc{Col}]{\includegraphics{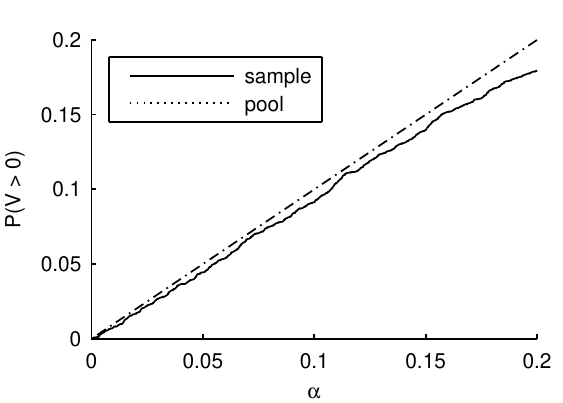}

}\subfloat[\textsc{Swap}]{\includegraphics{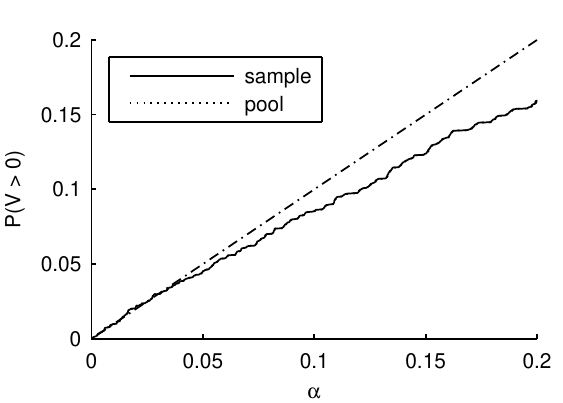}}\caption{\label{fig:99retail}
minP test for different $p$-value calculation and randomization methods
with $\retail$. The true FWER is plotted against the controlled level.
The diagonal dash-dotted line should not be exceeded for the minP-property
to be satisfied. }

\end{figure}

\end{document}